\documentclass[letterpaper, 10 pt, journal, twoside]{IEEEtran}  

\usepackage{lipsum}  

\usepackage{hyperref}

\usepackage{xcolor}

\usepackage{array}  

\usepackage{graphicx}  

\usepackage{amsmath}  
\usepackage{bm}
\usepackage{svg}
\usepackage{amsfonts} 
\usepackage{amssymb} 
\usepackage{mathtools}
\usepackage[acronym]{glossaries}
\usepackage{siunitx,booktabs}
\usepackage[export]{adjustbox}
\usepackage[noadjust]{cite}
\usepackage{mathtools}
\usepackage{wrapfig}
\usepackage[capitalise]{cleveref}

\usepackage{amsthm}

\usepackage[ruled,vlined,linesnumbered]{algorithm2e}

\SetCommentSty{mycommfont}
\SetKwComment{Comment}{$\triangleright$\ }{}
\let\oldnl\nl
\newcommand{\nonl}{\renewcommand{\nl}{\let\nl\oldnl}}
\crefname{assumption}{assumption}{assumptions}

\newtheorem{theorem}{Theorem}
\newtheorem{proposition}{Proposition}
\newtheorem{lemma}{Lemma}
\newtheorem{definition}{Definition}
\newtheorem{problem}{Problem}
\newtheorem{assumption}{Assumption}
\newtheorem{remark}{Remark}

\makeatletter
\newcommand\notsotiny{\@setfontsize\notsotiny\@vipt\@viipt}
\makeatother
\makeatletter
\newcommand{\removelatexerror}{\let\@latex@error\@gobble}
\makeatother

\DeclareMathOperator*{\argmin}{\arg\!\min}

\renewcommand{\H}[2][]{
\ifthenelse {\equal{#1}{}}
{\mathbb{H}\left[#2\right]}
{\mathbb{H}_{#1}\left[#2\right]}}

\newcommand{\E}[2][]{
\ifthenelse {\equal{#1}{}}
{\mathbb{E}\left[#2\right]}
{\mathbb{E}_{#1}\left[#2\right]}}

\newcommand{\norm}[1]{\left\lVert#1\right\rVert}

\def\Algref#1{Algorithm~\ref{#1}}

\def\1{\bm{1}}

\def\va{{\bm{a}}}
\def\vb{{\bm{b}}}
\def\vc{{\bm{c}}}
\def\vd{{\bm{d}}}

\def\vf{{\bm{f}}}
\def\vg{{\bm{g}}}
\def\vh{{\bm{h}}}

\def\vm{{\bm{m}}}

\def\vq{{\bm{q}}}

\def\vs{{\bm{s}}}

\def\vw{{\bm{w}}}

\def\vh{{\bm{h}}}

\def\mC{{\bm{C}}}

\def\mG{{\bm{G}}}

\def\mQ{{\bm{Q}}}

\def\mV{{\bm{V}}}

\DeclareMathAlphabet{\mathsfit}{\encodingdefault}{\sfdefault}{m}{sl}
\SetMathAlphabet{\mathsfit}{bold}{\encodingdefault}{\sfdefault}{bx}{n}

\def\gB{{\mathcal{B}}}
\def\gC{{\mathcal{C}}}

\def\gE{{\mathcal{E}}}
\def\gF{{\mathcal{F}}}
\def\gG{{\mathcal{G}}}

\def\gL{{\mathcal{L}}}
\def\gM{{\mathcal{M}}}

\def\gO{{\mathcal{O}}}
\def\gP{{\mathcal{P}}}

\def\gR{{\mathcal{R}}}

\def\gU{{\mathcal{U}}}
\def\gV{{\mathcal{V}}}

\def\sN{{\mathbb{N}}}

\def\sP{{\mathbb{P}}}

\def\sR{{\mathbb{R}}}

\newcommand{\defi}{\;{:=}\;}

\makeatletter
\newcommand{\StatexIndent}[1][3]{%
  \setlength\@tempdima{\algorithmicindent}%
  \Statex\hskip\dimexpr#1\@tempdima\relax}
\makeatother

\newacronym{poe}{PoE}{product of experts}
\newacronym{kld}{KL divergence}{Kullback–Leibler Divergence}
\newacronym{kl}{KL}{Kullback–Leibler}
\newacronym{map}{MAP}{maximum a Posteriori}
\newacronym{rmp}{RMP}{Riemannian Motion Policies}
\newacronym{gmm}{GMM}{Gaussian Mixture Model}
\newacronym{mpc}{MPC}{Model Predictive Control}
\newacronym{gp}{GP}{Gaussian Process}
\newacronym{ot}{OT}{Optimal Transport}

\newacronym{dmdp}{DMDP}{Deterministic Markov Decision Process}
\newacronym{mdp}{MDP}{Markov Decision Process}
\newacronym{vi}{VI}{Value Iteration}
\newacronym{dof}{DoF}{degrees of freedom}
\newacronym{sdf}{SDF}{Signed Distance Field}
\newacronym{mpot}{MPOT}{Motion Planning via Optimal Transport}
\newacronym{chomp}{CHOMP}{Covariant Hamiltonian Optimization for Motion Planning}
\newacronym{gpmp}{GPMP}{Gaussian Process Motion Planning}
\newacronym{stomp}{STOMP}{Stochastic Trajectory Optimization for Motion Planning}
\newacronym{sgpmp}{SGPMP}{Stochastic Gaussian Process Motion Planning}
\newacronym{bps}{BPS}{Batch Polytope Search}

\newcommand{\citet}{\cite}
\usepackage{setspace}
\bstctlcite{IEEEexample:BSTcontrol}

\title{
Global Tensor Motion Planning
}

\author{\small An T. Le$^{1}$, Kay Hansel$^{1}$, Jo\~{a}o Carvalho$^{1}$, Joe Watson$^{1,2}$, Julen Urain$^{1,5}$, Armin Biess, Georgia Chalvatzaki$^{1,5}$ and Jan Peters$^{1,2,3,4}$
\thanks{Manuscript received: 31.12.2024; Revised 10.04.2025; Accepted 20.05.2025.}
\thanks{This paper was recommended for publication by Editor Júlia Borràs Sol upon evaluation and Reviewers' comments.}
\thanks{
Corresponding author: An T. Le, \href{an@robot-learning.de}{an@robot-learning.de}
}
\thanks{
$^{1}$Intelligent Autonomous Systems Lab, TU Darmstadt, Germany;
$^{2}$German Research Center for AI (DFKI); 
$^{3}$Hessian.AI;
$^{4}$Centre for Cognitive Science.
$^{5}$Interactive Robot Perception \& Learning Lab, TU Darmstadt, Germany
}
\thanks{Digital Object Identifier (DOI): 10.1109/LRA.2025.3575307}
}

\begin{document}

\begin{titlepage}
\quad\\[1cm]
\makeatother
	{\Huge IEEE Copyright Notice}\\[0.5cm]
	{\begin{spacing}{1.2}
	\large \copyright \ 2025 IEEE. Personal use of this material is permitted. Permission from IEEE must be obtained for all other uses, in any current or future media, including reprinting/republishing this material for advertising or promotional purposes, creating new collective works, for resale or redistribution to servers or lists, or reuse of any copyrighted component of this work in other works.
	\end{spacing}}
\end{titlepage}

\markboth{IEEE ROBOTICS AND AUTOMATION LETTERS. PREPRINT VERSION. Accepted May, 2025}
{Le \MakeLowercase{\textit{et al.}}: Global Tensor Motion Planning} 

\maketitle

\begin{abstract}

Batch planning is increasingly necessary to quickly produce diverse and quality motion plans for downstream learning applications, such as distillation and imitation learning. This paper presents Global Tensor Motion Planning (GTMP)---a sampling-based motion planning algorithm comprising only tensor operations. We introduce a novel discretization structure represented as a random multipartite graph, enabling efficient vectorized sampling, collision checking, and search. We provide a theoretical investigation showing that GTMP exhibits probabilistic completeness while supporting modern GPU/TPU. Additionally, by incorporating smooth structures into the multipartite graph, GTMP directly plans smooth splines without requiring gradient-based optimization. Experiments on lidar-scanned occupancy maps and the MotionBenchMarker dataset demonstrate GTMP's computation efficiency in batch planning compared to baselines, underscoring GTMP's potential as a robust, scalable planner for diverse applications and large-scale robot learning tasks.

\end{abstract}


\section{Introduction}  \label{sec:intro}






Motion planning with probabilistic completeness has been a foundation of robotics research, with seminal works like PRM~\cite{kavraki1996probabilistic} and RRTConnect~\cite{kuffner2000rrt} serving as cornerstone methods for years~\cite{latombe2012robot}. However, as the complexity of robotic tasks increases, there is a growing demand for batch-planning methods. Several factors drive this interest: (i) the need to gather large datasets for policy learning~\cite{carvalho2023mpd, wang2021survey, reuss2023goal}, (ii) the inherent non-linearity of task objectives that lead to multiple viable solutions~\cite{le2024accelerating, mukadam2018continuous, osa2020multimodal}, and (iii) the increasing availability of powerful GPUs/TPUs for accelerated planning~\cite{bhardwaj2022storm, sundaralingam2023curobo}. Despite these advances, batching traditional sampling-based planners, such as RRT/PRM and their variants, remains an ongoing challenge~\cite{pan2012gpu, bialkowski2011massively, blankenburg2020towards, jacobs2012scalable}. Their underlying discretization techniques, such as the incremental graph construction of RRT/PRM or the search mechanism of A*~\cite{hart1968astar, russell2016artificial}, are not conducive to efficient vectorization over planning instances.

This paper revisits classical motion planning, where we plan from a single start configuration to multiple goal configurations. We introduce a simple yet effective discretization structure with layers of waypoints, which can be represented as tensors, enabling GPU/TPU utilization.
We propose Global Tensor Motion Planning (GTMP), which enables highly batchable operations on multiple planning instances, such as batch collision checking and batch~\gls{vi}, inducing an easily vectorizable implementation with JAX~\cite{jax2018github}. This simplicity allows for differentiable planning and rapid integration with modern frameworks, making the algorithm particularly desirable for real-time applications and scalable data collection for robot learning. Our experimental results demonstrate much better batch efficiency planning than standard baseline implementations while achieving similar smoothness and better path diversity with the spline discretization structure. 

Our contributions are twofold: i) we propose a \textit{vectorizable sampling-based planner} exhibiting probabilistic completeness, which does not require simplification routines~\cite{sucan2012open}, and ii) we extend GTMP with a spline discretization structure, enabling batch spline planning with path quality comparable to trajectory optimizers.

\begin{figure*}[th]
    \centering
    \includegraphics[width=\textwidth]{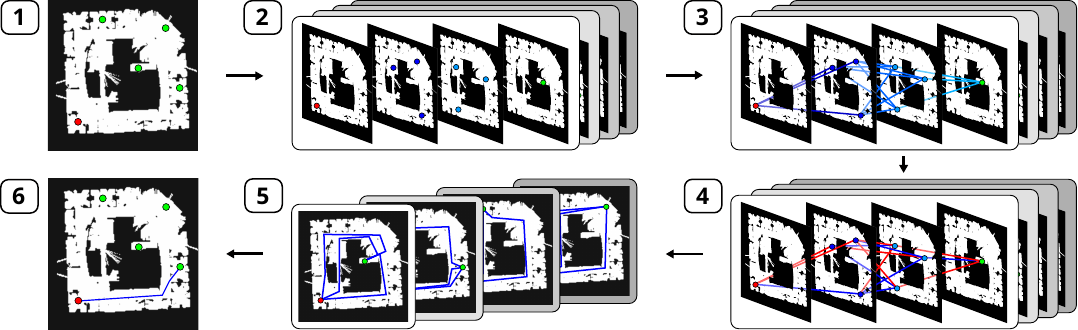}
    \vspace{-0.5cm}
    \caption{GTMP can plan with multiple goals or \texttt{vmap} over goals. For clarity, we present an example of performing JAX \texttt{vmap} on GTMP (M=2, N=3) over the batch of $B=3$ seeds. \textbf{(1)} The objective is to find a batch of feasible paths from the start (red) to the goals (green). \textbf{(2, 3)} In each seed, we sample a multipartite graph and form a tensor (\Algref{alg:gtmp}, Line 1). \textbf{(4)} A batch of collision checks is performed and stored into cost matrices (\Algref{alg:gtmp}, Line 2). \textbf{(5)} Then, per seed, we execute finite value iterations (\Algref{alg:gtmp}, Line 5-7) and trace the optimal path from the optimal value matrices (\Algref{alg:gtmp}, Line 8-13). \textbf{(6)} For execution, we can select the best path in terms of exemplary shortest path criteria. More information can be found on \url{https://sites.google.com/view/gtmp}.}
    \label{fig:method}
\end{figure*}

\section{Related Works} \label{sec:related_works}

Vectorizing motion planning has been an active research topic for decades. Here, we briefly survey the most relevant works on vectorizing either at the \textit{algorithmic}-level (e.g., collision-checking) or \textit{instance}-level (e.g., batch trajectory planning).

\textbf{Sampling-based Vectorization.} Recognizing the importance of planning parallelization, earliest works \cite{amato1999probabilistic, plaku2005sampling, bialkowski2011massively, jacobs2012scalable, pan2012gpu} propose a \textit{vectorizable} collision-checking data structure. State-of-the-art work on leveraging CPU-based \textit{single instruction, multiple data}~\cite{thomason2024motions} (i.e., VAMP) has pushed collision-checking efficiency to \textit{microseconds}. In a different vein, a body of works \cite{gammell2020batch, janson2015fast, strub2020adaptively, wang2020neural, yu2021reducing, ichter2018learning} proposes a learning heuristic or batch-sampling strategies to inform or refine the search-graph with new samples, effectively reducing collision checking.
Despite the hardware or algorithmic acceleration efforts, past works still resort to discretization structures such as trees for RRT variants or graphs for PRM variants \cite{orthey2023sampling}, which are unsuitable for instance-level vectorization. 

\textbf{Vectorizing Trajectory Optimization.} Vectorizing optimization-based planner with GPU-acceleration~\cite{adajania2022multi, lambert2020stein, bhardwaj2022storm, urain2022sgmpg, le2024accelerating} gained traction recently due to their computational efficiency, the solutions' multi-modality, and their robustness to bad local-minima. However, these local methods are sensitive to initial conditions and may get stuck in large infeasible regions, thereby the need for warmstarting the sampling-based global solutions~\cite{sundaralingam2023curobo}. GTMP addresses this issue by proposing a layerwise discretization structure, enabling vectorization in sampling and search operations while having better global solutions.

 






\section{Tensorizing Motion Planning} \label{sec:method}

We consider the path planning problem~\cite{lavalle2006planning} for a configuration $\vq$ in compact space $\vq\in \gC \subset \sR^d$ having $d$-dimensions, with $\gC_{\textrm{coll}}$ being the collision space such that $\gC \setminus \gC_{\textrm{coll}}$ is open. Let $\gC_{\textrm{free}} = \mathsf{Cl}(\gC \setminus \gC_{\textrm{coll}})$ be the free space, with $\mathsf{Cl}(\cdot)$ the set closure. Denote the start configuration $\vq_0$ and a set of goal configurations $\gG$. Let $f: [0, 1] \rightarrow \vq,\,\vf(t) \in \gC$, we can define its total variation as its arc length
\begin{equation} \label{eq:path_tv}
    \textrm{TV}(f) = \sup_{M \in \sN,0=t_0,\ldots,t_M=1} \textstyle\sum_{i=1}^M \norm{\vf(t_i) - \vf(t_{i-1})},
\end{equation}
\begin{definition}[Feasible Path]
    The function $f: [0, 1] \rightarrow \vq$ with $\textrm{TV}(f) < \infty$ is 
    \begin{itemize}
        \item a path, if it is continuous.
        \item a feasible path, if and only if $\forall t \in [0, 1], \vf(t) \in \gC_{\textrm{free}},\, \vf(0) = \vq_0,\, \vf(1) \in \gG$.
    \end{itemize}
\end{definition}
\noindent Let $\gF$ be the set of all paths. We denote $\gF_{\textrm{free}}$ as the set of feasible paths for a feasible planning problem. Here, we do not consider dynamic constraints, invalid configurations that violate collision constraints, and configuration limits. 
\begin{problem}[Batch Path Planning]
    Given a planning problem $(\gC_{\textrm{free}}, \vq_0, \gG)$ and cost function $c: \gF \rightarrow \sR_{>0}$, find a batch of $B > 0$ feasible path $f$ and report failure if no feasible path exists.
\end{problem}
\noindent This problem definition is standard for several robotic settings, such as serial manipulators with joint limits. We propose to solve Problem 1 with probabilistic completeness, striving to discover multiple solution modes.

\textbf{Practical Motivation.} In essence, GTMP leverages a fixed discretization structure over the whole search space, represented by fixed-shape tensors, to enable efficient planning vectorization with JAX \texttt{vmap} operation~\cite{jax2018github}. This approach contrasts with the incremental discretization structures of classical motion planning algorithms, which procedurally expand the search space during planning. 

\subsection{Discretization Structure}
We introduce the random multipartite graph as a novel configuration discretization structure designed to represent planning problems as tensors.

\begin{definition}[Random Multipartite Graph Discretization] \label{def:graph}
    Consider a geometric graph $G = (\gV, \gE)$ on configuration space $\gC$, the node set $\gV$ is represented by $\{\vq_{\textrm{s}}, \gM, \gG\}$, where $\gM = \{\gL_m\}_{m=1}^M$ is a set of $M$ layers. Each layer $\gL_m = \{\vq_i \in \gC \mid \vq_i \sim p_m\}_{i=1}^N$ contains $N$ waypoints sampled by an associated proposal distribution $p_m$ on $\gC$. The edge set $\gE$ is defined by the union of (forward) pair-wise connections between the start and first layer $\{ (\vq_{\textrm{s}}, \vq) \mid \forall\vq \in \gL_1 \},$ between layers in $\gM$
    \begin{equation}
        \{ (\vq_m, \vq_{m+1}) \mid \forall \vq_m \in \gL_m,\ \vq_{m+1} \in \gL_{m + 1},\, 1 \leq m < M \}, \nonumber
    \end{equation}
    and between the last layer and goals $\{ (\vq, \vq_{\textrm{g}}) \mid \forall \vq \in \gL_M,\ \vq_{\textrm{g}} \in \gG\},$
    leading to a complete $(M+2)$-partite directed graph.
\end{definition}
We typically set $p_m = \gU(\gC)$ as uniform distributions over configuration space (bounded by configuration limits cf.~\cref{fig:method}). Consequently, the graph nodes are represented as the waypoint tensors for all layers $\mQ \in \sR^{M \times N \times d}$ and the goal configuration $\mG \in \sR^{|\gG| \times d}$ from $\gG$, within the state limits. Extending~\cref{def:graph} to \textit{spline discretization structure} by replacing the straight line with the cubic polynomials, representing any edge $(\vq, \vq') \in \gE$, is straightforward with Akima spline~\cite{akima1974method} (cf.~\cref{sec:akima}).

\begin{definition}[Path In $G$]
    A path $f: [0, 1] \rightarrow \vq$ in $G$ exists if it $\vf(0) = \vq_0,\, \vf(1) \in \gG$ and its piecewise linear segments correspond to edges connecting $\vq_0$ and $\vq_g \in \gG$.
\end{definition}

\subsection{State Machine On Graph}
The graph $G$ is represented by the state machine $(\gV, \gE, c, t)$~\cite{puterman2014markov}, where the state set is the node set of $G$, the action set is equivalent to the edge set $\gE$, the transition cost function $c: \gV \times \gE \rightarrow \sR$, deterministic state transition $t(\vq' \mid \vq, (\vq, \vq')) = 1,\, (\vq, \vq') \in \gE$.
The goal set $\gG \subset \gV$ is the terminal set with terminal costs $c_g(\vq),\,\vq \in \gG$. A policy $\pi: \gV \rightarrow \gE$ depicts the decision to transition to the next layer, given the current state at the current layer.

We use unbounded occupancy collision costs
\begin{equation} \label{eq:cost_coll}
    c_{\textrm{coll}}(\vq) = 0 \textrm{ if } \vq \in \textrm{int}_{\delta}(\gC_{\textrm{free}}) \textrm{, else } \infty,
\end{equation}
which merges the planning and verification steps (cf. Proposition~\ref{prop:feasible}).
Then, the transition cost function can be defined
\begin{equation} \label{eq:cost}
   c(\vq, (\vq, \vq')) = \underbrace{\int_{a}^{b} c_{\textrm{coll}}(\vf(t)) f' d t}_{\textrm{collision}} + \underbrace{\norm{\vq - \vq'}}_{\textrm{smoothness}},
\end{equation}
where the collision term is a straight-line integral with $f' = 1 / \norm{\vq' - \vq}$ between $\vf(a) = \vq$ and $\vf(b) = \vq'$.
Finding the optimal value function on $G$ is straightforward by iterating the Bellman optimality operator
\begin{align} \label{eq:bellman}
    v_G(\vq) \leftarrow &\min_{(\vq, \vq')} \sum_{\vq'} t(\vq' \mid \vq, (\vq, \vq')) \left(c(\vq, (\vq, \vq')) + v_G(\vq') \right) \nonumber \\
   \leftarrow &\min_{(\vq, \vq')} \left(c(\vq, (\vq, \vq')) + v_G(\vq') \right)
\end{align}
with a finite number of iterations $K = M + 1$. The optimal policy is extracted by tracing the optimal value function
\begin{equation} \label{eq:trace}
    \pi^*(\vq) = \argmin_{(\vq, \vq')} \left(c(\vq, (\vq, \vq')) + v_G^*(\vq') \right),
\end{equation}
from $\vq_0$ until $\vq' \in \gG$~\cite{puterman2014markov}. This produces a sequence of edges ${\gP = \{(\vq_0, \vq_1), \ldots, (\vq_M, \vq_g) \mid \vq_g \in \gG\}}$.
\begin{proposition} \label{prop:clen}
    By following any policy on $(\gV, \gE, c, t)$ from $\vq_0$, $\gP$ has a constant cardinality of $M + 1$.
\end{proposition}
\begin{proof}
    By construction of graph $G$, each application of~\cref{eq:trace} increases the layer number $m$ strictly monotonically, since $t(\vq_{m+1} \mid \vq_m, \pi(\vq_m)) = t(\vq_{m+1} \mid \vq_m, (\vq_m, \vq_{m+1})) = 1,\, (\vq_m, \vq_{m+1}) \in \gE$. Hence, $|\gP| = M + 1$.
\end{proof}
Finding optimal paths by finite~\gls{vi} over a discretization structure has been a common practice and widely applied in different settings~\cite{bertsekas2012dynamic}. However, to our knowledge, applying~\gls{vi} over a random multipartite graph, enabling batching mechanisms over planning instances, is novel, as we present in the next section.

\subsection{Batching The Planner}

In practice, we do not need to construct an explicit graph data structure due to $G$'s multipartite structure.
Observing the deterministic state transition and the equal cardinality of layers, we just need to compute and maintain the transition cost matrices $\mC_s \in \sR^N,\, \mC_h \in \sR^{(M - 1) \times N \times N},\, \mC_l^{N \times |\gG|}$ and value matrices $V_s \in \sR,\, \mV_h \in \sR^{M \times N},\, \mV_g \in \sR^{|\gG|}$, where $\mC_s$ is the transition costs from $\vq_0$ to the first layer; $\mC_h, \mC_l$ hold transition costs between middle layers and last layer to goals; $V_s, \mV_h, \mV_g = \mC_g$ hold values of start, layers costs and terminal goal costs. Given the uniformly-sampled waypoint tensors $\mQ \in \sR^{M \times N \times d}$ and the goals $\mG \in \sR^{|\gG| \times d}$, the cost-to-go term of the transition costs~\cref{eq:cost} is approximately computed by first probing an $H$ number of equidistant points on all edges, evaluating them in batches, and taking the mean values over the probing dimension. We assume all cost functions are batch-wise computable. 

The GTMP algorithm is compactly presented in~\Algref{alg:gtmp}. Note that Line 6 is a matrix-reduced \texttt{min} operation on the last dimension, while the \texttt{sum} is broadcasted to the middle dimension of the cost matrix $\mC_h \in \sR^{(M - 1) \times N \times N}$ from the value matrix $\mV_h \in \sR^{M \times 1 \times N}$. After $M + 1$ Bellman iterations (Line 5-7), given the converged value matrix $\mV_h^*$, a sequence of waypoints is traced over the layers to the goals (Line 11-13). Notice that all component matrices can be straightforwardly vectorized by adding the batch dimension $B$ for all matrices, and the whole algorithm can be JAX \texttt{vmap} over sampling seeds on line 1. Note that~\cite{pan2012gpu, bialkowski2011massively, blankenburg2020towards, thomason2024motions} focus on vectorizing collision checking or forward kinematics in a single planning instance, while we can ensure that~\Algref{alg:gtmp} can be vectorized at the instance-level~\cite{le2024accelerating} by~\cref{prop:clen}. 
\begin{figure}[b]
\vspace{-0.5cm}
\removelatexerror
\begin{algorithm}[H]
\caption{Global Tensor Motion Planning}
\label{alg:gtmp}
\DontPrintSemicolon
\KwIn{Start $\vq_0$, Goals $\mG \in \sR^{|\gG| \times d}$}
Uniformly sample $\mQ \in \sR^{M \times N \times d}$ \\ 
Compute cost matrices $\mC_s, \mC_h, \mC_l$ as~\cref{eq:cost} \\ 
Init $V_s \in \sR,\, \mV_h \in \sR^{M \times N},\, \mV_g \in \sR^{|\gG|}$ \tcp*{finite VI}
    \For{$1 \leq k \leq M + 1$}{
        $\mV_h[M-1] \leftarrow \min (\mC_l + \mV_g)$ \\
        $\mV_h[:M-1] \leftarrow \min (\mC_h + \mV_h[1:],\, \textrm{axis}=-1)$ \\
        $V_s \leftarrow \min (\mC_s + \mV_h[0]) $\\
     }
     $i \leftarrow \argmin (\mC_s + \mV_h^*[0])$ \tcp*{path tracing}
     $\gP = \{i\}$ \\
     \For{$1 \leq m  \leq M - 1$}{
        $i \leftarrow \argmin (\mC_h[m - 1, i] + \mV_h^*[m])$ \\
        Append $\mQ[m, i]$ to $\gP$\\
     }
    $i \leftarrow \argmin (\mC_l[i] + \mV_g)$ and append $\mG[i]$ to $\gP$ \\
    \KwOut{$\gP$}
\end{algorithm}%
\vspace{-0.5cm}
\end{figure}


\textbf{Complexity Analysis.} The Bellman matrix update (Line 5-7) is an asynchronous update in batches (i.e., updates based on values of previous iteration) and also known to converge~\cite{bertsekas2015parallel}. Considering the layer number $M$, waypoint number per layer $N$, and probing number $H$, we assume that the Bellman matrix update is executed on $P$ processor units, an estimate of time complexity per~\gls{vi} iteration is $\gO(MN^2 / P)$ due to the broadcasted \texttt{sum} and \texttt{min} operator on Line 6~\Algref{alg:gtmp}. Hence, the overall worst-case time complexity is $\gO(M^2N^2/P)$, with a fixed number of $M+1$ \gls{vi} iterations.
The collision-checking time complexity is $\gO(MN^2H/P)$, and thus, the overall time complexity is $\gO(MN^2(H + M)/P)$. The space complexity is $\gO(MN^2H)$ due to the collision checking. Theoretical investigations regarding GTMP probabilistic completeness are presented in~\cref{sec:theory}.


\section{Experiment Results} \label{sec:experiments}

We assess the performance of GTMP and its smooth extension on batch planning and single planning capability compared to popular baselines and collision-checking mechanisms. Hence, we investigate the following questions for batch trajectory generation, or for finding the global solution: i) how does GTMP with JAX/GPU-implementation compare to highly optimized probabilistic-complete planners implemented in PyBullet/OMPL~\cite{sucan2012open, coumans2019} or in VAMP~\cite{thomason2024motions}?, ii) how does GTMP-Akima compare to popular gradient-based smooth trajectory optimizers such as CHOMP~\cite{zucker2013chomp} or GPMP~\cite{mukadam2018continuous}?, and iii) Are the empirical results consistent with the theoretical guarantees (\cref{thm:prob})?

\textbf{Settings.} We run all CPU-based planners (RRTC, BKPIECE) on AMD Ryzen 5900X clocked at 3.7GHz and GPU-based planners (GTMP, CHOMP, GPMP, cuRobo) on a single Nvidia RTX 3090. Note that GTMP, CHOMP, and GPMP are implemented in JAX~\cite{jax2018github}, and the planning times are measured after JIT. We use cuRobo's official PyTorch implementation.  We initialize CHOMP and GPMP with samples from a high-variance Gaussian process prior~\cite{urain2022sgmpg} connecting from the start to the goals. We set a default probing $H=10$ and used uniform sampling for all GTMP runs. For all CPU-based planners, we give a timeout of one minute and report metrics after simplification routines. Planning time per task is the sum of all planning instances, which includes simplification time for CPU-based planners, while GTMP does not need path simplification.

\textbf{Metrics.} The metrics are chosen for comparing across probabilistically-complete planners and trajectory optimizers: (i) \textit{Planning Time (s)} in seconds of a batch of paths given a task, (ii) \textit{Collision Free (CF \%)} percentage of paths in a batch (failure cases are either in collision or timeout), (iii) \textit{Minimum Cosine Similarities (Min Cosim)} over consecutively path segments and averaging over the batch of paths in a task, (iv) \textit{Paths Diversity (PD)} as the mean of pairwise Sinkhorn~\cite{cuturi2013sinkhorn} distances in a batch having $B$ paths
\begin{equation}
    PD = \frac{1}{B (B - 1)}{\textrm{OT}_\lambda (\gP_i, \gP_j)},\, i,j \in \{1, \ldots, B\},
\end{equation}
where we treat the path $\gP = \{\vq_0, \ldots \vq_T\}$ as empirical distribution with uniform weights, and different paths can have different horizons $T$. The entropic scalar $\lambda=5e^{-3}$ is constant. The metric \textit{Min Cosim} measures the worst/average rough turns over path segments, which represents worst-case jerks since baselines plan different trajectory dynamic orders. The PD measures the spread of solution paths correlating to solutions' modes discovery. 

\subsection{Batch Planning Comparison} \label{subsec:batch_exp}

 Fig.~\ref{fig:exp} (top-row) compares GTMP and GTMP-Akima with OMPL implementation of (single-query) RRTConnect~\cite{kuffner2000rrt} and BKPIECE~\cite{csucan2009kinodynamic}. The environments are planar occupancy maps of Intel Lab, ACES3 Austin, Orebro, Freiburg Campus, and Seattle UW Campus generated from the Radish dataset~\cite{radish}. The maps are chosen to include narrow passages, large spaces, and noisy occupancies (cf.~\cref{fig:method}). We randomly sample $100$ start-goal pairs as tasks on each map and plan $100$ paths per task. We clearly see a comparable Min Cosim (i.e., similar statistics of rough turns) and PD of GTMP (M=200, N=4) compared to baselines across maps and in aggregated statistics over maps. With JIT and GPU utilization, GTMP consistently produces batch paths with a fixed number of segments and x10000 less wall-clock time compared to baselines across maps.
\begin{figure*}[t]
    \centering
    \includegraphics[width=\linewidth]{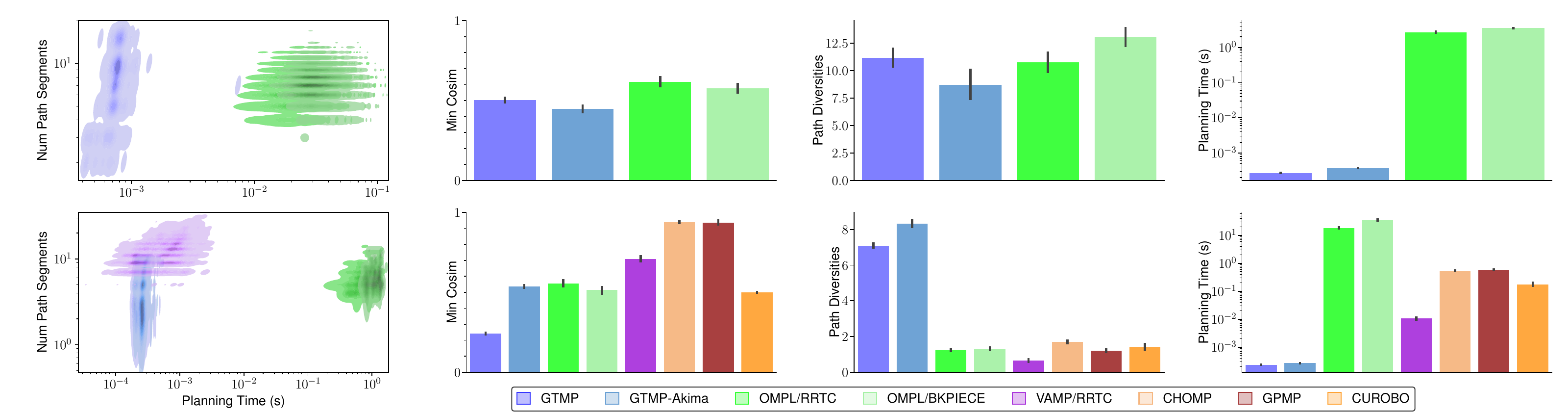}
    \vspace{-0.7cm}
    \caption{Aggregated statistics of comparison experiments on Planar Occupancy (top-row) and Panda MBM dataset (bottom-row). We note the log scale on the Planning Time axes. The batch planning time is the sum of instance time for sequential planners (last column). All plotted data points are based on successful path statistics.}
    \label{fig:exp}
\end{figure*}
    
    
    
    


\begin{table}[b]
 \vspace{-0.5cm}
    \addtolength{\tabcolsep}{-0.1em}
    \tiny    
    \centering
    \caption{Aggregated Statistics Of M$\pi$Nets Dataset}
    \label{tab:mpinet_exp}
    
    \begin{tabular}{l|ccccc}

    \textbf{Algorithms} & \textbf{PT $\downarrow$} (ms) & \textbf{Success $\uparrow$} (\%) & \textbf{Path Length $\downarrow$} & \textbf{Min Cosim $\uparrow$} & \textbf{PD $\uparrow$} \\
    \midrule
    GTMP (N=30, M=2) & $0.11$ & $99.6$ & $4.8$ & $-0.3$ & $7.7$ \\
    GTMP-Akima (N=30, M=2) & $0.10$ & $97.1$ & $7.3$ & $-0.1$ & $\mathbf{7.8}$ \\
    \midrule
    VAMP/RRTC~\cite{thomason2024motions} & $\mathbf{0.09}$ & $\mathbf{100.0}$ & $3.6$ & $\mathbf{0.1}$ & - \\
    cuRobo~\cite{sundaralingam2023curobo} & $43.1$ & $99.7$ & $\mathbf{2.8}$ & $0.0$ & - \\

    \end{tabular}
\end{table}
We choose the \textsc{MotionBenchMaker} (MBM) dataset~\cite{chamzas2021motionbenchmaker} of $7$-DoF Franka Emika Panda tasks such as table-top manipulation (\textit{table pick and table under pick}), reaching (\textit{bookshelf small, tall and thin}), and highly-constrained reaching (\textit{box and cage}). Each task is pre-generated with $100$ problems available publicly.
We implement our collision-checking in JAX via primitive shape approximation, such as a Panda spherized model, oriented cubes, and cylinders representing tasks in MBM. The default hyperparameters and compilation configurations for VAMP/RRTC, OMPL/RRTC, and OMPL/BKPIECE are also adopted following~\cite{thomason2024motions}\footnote{We use default \textit{shortcut} simplification for OMPL planners while using default \textit{shortcut and B-spline smoothing} for VAMP/RRTC.}. CHOMP and GPMP plan first-order trajectories having a horizon of $T=32$. All algorithms are compared on the planning performance of a batch of $B=50$ paths for all tasks.

Fig.~\ref{fig:exp} (bottom-row) shows the planning performance comparisons between GTMP (N=30, M=2) and baseline probabilistically-complete planners and gradient-based trajectory optimizers. We see that GTMP consistently has the best diversity (PD) and worst rough turn statistics (Min Cosim) in all tasks. This is due to the maximum exploration behavior of GTMP by sampling uniformly over configuration space, which increases the risk of rough paths. In principle, increasing points per layer $N$ while having minimum solving layers $M$ would improve Min Cosim due to having more chances to discover smoother paths with fewer segments, as long as GPU memory allows (cf.~\cref{fig:sweep_exp}). Compared with gradient-based optimizers, GTMP-Akima with spline discretization construction has a similar Min Cosim to cuRobo while not requiring gradients from the planning costs. Note that cuRobo additionally considers dynamical constraints, which increases planning time but improves metrics such as maximum model jerk. On batch planning efficiency, GTMP and GTMP-Akima achieve x50 faster than state-of-the-art VAMP/RRTC implementation while being x2500 faster than CHOMP/GPMP/cuRobo and x100000 faster than the OMPL implementation with PyBullet collision checking. We leave the investigation of combining GTMP with the VAMP collision checking for future work.

Fig~\ref{fig:exp} (first-column) shows the distributions of single-instance planning time versus number of path segments, reflecting inherent algorithmic differences between GTMP and RRTC implementations. RRTC blobs are spread due to differences in randomized graph explorations between planning instances and are separated due to differences in collision-checking efficiency~\cite{thomason2024motions}. GTMP vectorizes planning via layered structure, resulting in predictable narrow distribution due to fixed-segment path planning.

\subsection{Single Plan Comparison} \label{subsec:mpinet_exp}

We compare GTMP and GTMP to the strong baselines such as VAMP/RRTC~\cite{thomason2024motions} and cuRobo~\cite{sundaralingam2023curobo}, in terms of single planning for execution, on the M$\pi$Nets dataset~\cite{fishman2023motion} of diverse $7$-DoF Franka Emika Panda tasks. We set $B=50$ for GTMP/GTMP-Akima and select the lowest path length for execution. We plan a single instance for VAMP/RRTC and cuRobo.
~\cref{tab:mpinet_exp} shows that GTMP achieves a similar success rate to the baselines (i.e., at least one successful path in the batch) while having similar planning time to the state-of-the-art VAMP/RRTC. However, due to the maximum exploration nature, GTMP performs worse regarding path quality. Future works on better sampling strategy per layer could improve GTMP path quality while increasing the sample efficiency on $M, N$ for low-memory planning.


    

    


\subsection{Ablation Study} \label{subsec:ablation}
This section explores various aspects of GTMP by sweeping the number of layer $M$ and number of points per layer $N$.~\cref{fig:sweep_exp} shows the sweeping statistics of $M \in \{2, 3, \ldots, 80\}, N\in \{10, 11, \ldots, 100\}$ on the Intel Lab occupancy map with a fixed start-goal pair to experimentally confirm the probabilistic completeness~\cref{thm:prob}. In Fig.~\ref{fig:sweep_exp}, Planning Time heatmap shows an experimentally infinitesimal increase in polynomial planning time-complexity over increasing $M,N$ (due to JIT-ing finite VI loops and efficient batch collision-checking, cf.~\cref{sec:method}). Then, the CF(\%) heatmap directly reflects the path existence probability~\cref{eq:probc}. Notice that the minimum layer $M_m = 3$ must be set for collision-free paths in the batch. Interestingly, $M_m$ is also the optimal number of layers to achieve non-zero CF(\%) with a minimal point per layer $N$ (red star), which confirms the observation in~\cref{sec:theory}. Next, further observations on Min Cosim also confirm that with less $M$, the paths are smoother. Finally, higher path diversity is induced by having higher CF(\%), corresponding to the top-right heatmap.

\begin{figure}[t]
    \centering
    \includegraphics[width=\columnwidth]{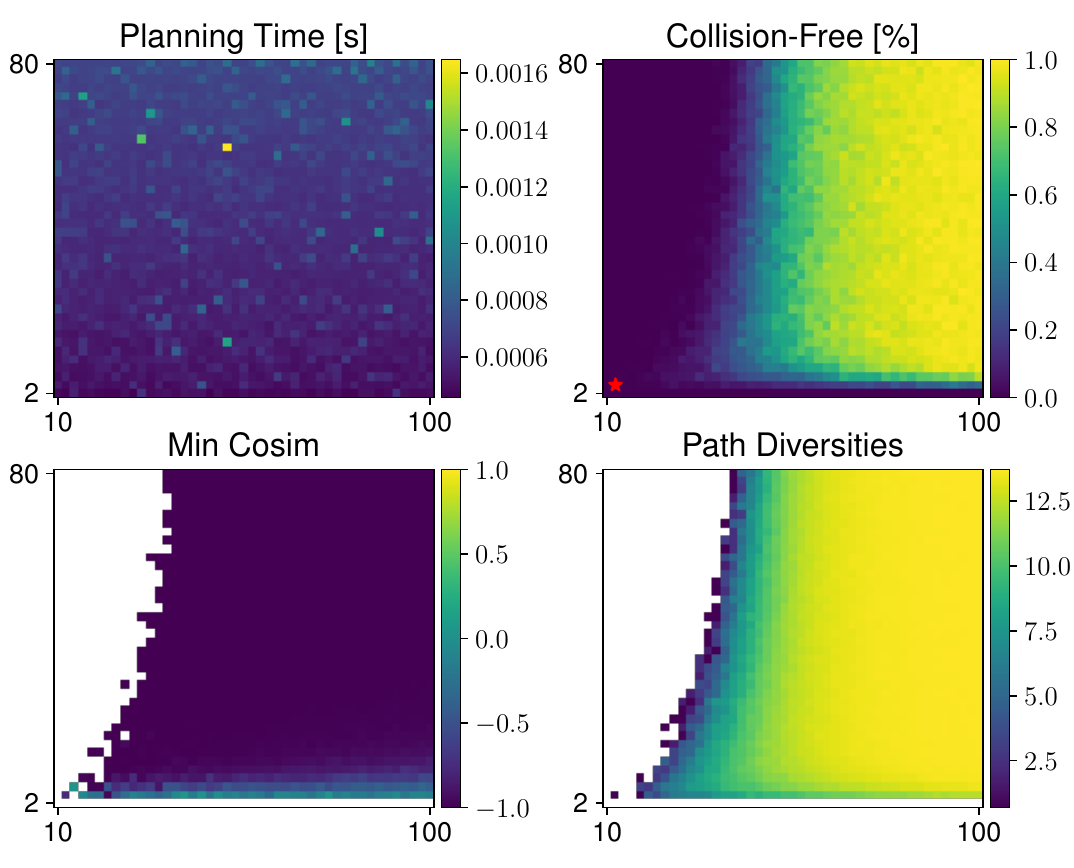}
    \vspace{-0.8cm}
    \caption{For each $M$ (y-axis), $N$ (x-axis), we set the number of probing $H=30$ and plan the batch of $B=200$ paths. The red star denotes the minimum number of layers $M_m$, corresponding to the minimum requirement of $N$ to discover some solutions experimentally.}
    \label{fig:sweep_exp}
    \vspace{-0.6cm}
\end{figure}

\section{Discussion \& Conclusions} \label{sec:conclusion}

GTMP offers several advantages algorithmically, as it is \textit{vectorizable} over a large number of planning instances, it does not require \textit{joint-limit enforcement} (i.e., sampling points in the limits), \textit{gradients} or \textit{simplification routines}. On the practical side, GTMP is \textit{easy to implement} (i.e., only tensor manipulation), \textit{easy to tune} (i.e., hyperparameter $M, N, H$), and \textit{easy to incorporate} motion planning objectives in~\cref{eq:cost_akima}.

GTMP is designed to be efficient in batch planning representing multiple instances of the same planning problem. The batch dimension representing the multiple GTMP planning instances can be interpreted as multiple replanning attempts. Indeed, Theorem~\ref{thm:prob} depicts probabilistic completeness over the batch dimension and $M, N$, contrasting with probabilistic completeness over exploration nodes as in RRT*~\cite{karaman2011sampling}. Beyond GTMP, since~\Algref{alg:gtmp} is cheap in common case, we could also derive an outer loop gradually increasing $M, N$ until some solutions in the batch are found.

GTMP addresses global exploration challenges but comes with memory requirements, especially for GPU acceleration. In contrast, local methods such as CHOMP or GPMP leverage gradient-based, more memory-efficient trajectory optimization. GTMP-Akima, for instance, avoids the need for gradients while delivering smooth velocity trajectories by a \textit{spline discretization structure}, making it a viable initialization for methods like GPMP, potentially combining the strengths of both approaches.

Variants of GTMP emphasize maximum exploration while maintaining smooth trajectory structures. Exploring further smooth discretization structures for higher-order planning is exciting, as the current Akima discretization structure only provides a $C^1$ spline grid. Furthermore, we are eager to adopt the efficient collision-checking of VAMP~\cite{thomason2024motions} for GTMP, when the VAMP batching configuration collision-checking becomes available, extending GTMP to CPU-based vectorization. Lastly, GTMP suggests the direction of probabilistically-complete batch planners, serving as a differentiable global planner or a competent oracle for learning.




\section*{APPENDIX}
\label{app:akima}
\section{Extension: Akima Spline} \label{sec:akima}

The Akima spline~\cite{akima1974method} is a piecewise cubic interpolation method that exhibits $C^1$ smoothness by using local points to construct the spline, avoiding oscillations or overshooting in other interpolation methods, such as cubic splines or B-splines.

\begin{definition}[Akima Spline] \label{def:akima}
    Given a point set $\{ \vq_i | \vq \in \gC \}_{i=1}^P $,  the Akima spline constructs a piecewise cubic polynomial $ f(t) $ for each interval \( [t_i, t_{i+1}] \) 
\begin{equation} \label{eq:poly_akima}
    f_i(t) = \vd_i (t - t_i)^3 + \vc_i (t - t_i)^2 + \vb_i (t - t_i) + \va_i,
\end{equation}

where the coefficients $\va_i, \vb_i, \vc_i, \vd_i \in \gC$ are determined from the conditions of smoothness and interpolation. Let $ \vm_i = (\vq_{i + 1} - \vq_{i}) / (t_{i+1} - t_i) $ at $ t_i $, the spline slope is computed from $ m_{i-1}, m_{i+1} $

\begin{equation}
\vs_i = \frac{|\vm_{i+1} - \vm_i | \vm_{i-1} + |\vm_{i-1} - \vm_{i-2} | \vm_{i}}{|\vm_{i+1} - \vm_i| + |\vm_{i-1} - \vm_{i-2}|}.
\end{equation}

The spline slopes for the first two points at both ends are $\vs_1 = \vm_1, \vs_2 = (\vm_1 + \vm_2) / 2, \vs_{P-1} = (\vm_{P-1} + \vm_{P-2}) / 2, \vs_P = \vm_{P-1}$. Then, the polynomial coefficients are uniquely defined 
\begin{align} \label{eq:akima_coeff}
&\va_i = \vq_i, \quad \vb_i = \vs_i, \nonumber\\
&\vc_i = (3\vm_i - 2\vs_i - \vs_{i+1}) / (t_{i+1} - t_i),\\
&\vd_i = (\vs_i + \vs_{i+1} - 2\vm_{i}) / (t_{i+1} - t_i)^2. \nonumber
\end{align}
\end{definition}

The Akima spline slope is determined by the local behavior of the data points, preventing oscillations that can occur when using global information. Interpolating with Akima spline does not require solving large systems of linear equations, making it computationally efficient as an ideal extension to~\cref{def:graph} to \textit{a spline discretization structure}.

\begin{definition}[Akima Spline Graph]
    Given a geometric graph $G = (\gV, \gE)$ (cf.~\cref{def:graph}), the Akima Spline graph $G_A$ has the edge set $\gE$ geometrically augmented by cubic polynomials. In particular, consider an edge $(\vq_{m, i}, \vq_{m+1, j}) \in \gE$ with $i, j$ are respective indices of points at layers $\gL_m, \gL_{m+1}$, the spline slope is defined with $ \vm_{m, i, j} = (\vq_{m+1, j} - \vq_{m, i}) / (t_{i+1} - t_i) $ as Modified Akima interpolation~\cite{akima1974method}
    {
    \begin{align}
        \vs_{m, i, j} &= \frac{\vw_{m, i, j} \vm_{m-1, i, j} +  \vw_{m-1, i, j} \vm_{m, i, j}}{\vw_{m, i, j} + \vw_{m-1, i, j}} \nonumber\\ 
        \vw_{m, i, j} &= \left|\frac{1}{N^2} \sum_{i, j} \vm_{m+1, i, j} - \vm_{m, i, j} \right| \\
        &+ \frac{1}{2} \left|\frac{1}{N^2} \sum_{i, j} \vm_{m+1, i, j} + \vm_{m, i, j} \right| \nonumber\\
        \vw_{m-1, i, j} &= \left|\vm_{m - 1, i, j} -  \frac{1}{N^2} \sum_{i, j} \vm_{m - 2, i, j} \right| \nonumber\\
        &+ \frac{1}{2} \left|\vm_{m - 1, i, j} + \frac{1}{N^2} \sum_{i, j} \vm_{m - 2, i, j} \right| \nonumber.
    \end{align}
    }%
    Then, the augmented cubic polynomial $f_{i, j}(t),\, t \in [t_m, t_{m+1}]$ is computed following~\cref{eq:akima_coeff}
    \begin{align} \label{eq:akima_segment}
    &\vs_m = \frac{1}{N^2} \sum_{i, j} \vs_{m, i, j},\,\va_{m, i, j} = \vq_{m, i},\, \vb_{m, i, j} = \vs_m,\\
    &\vc_{m, i, j} = (3\vm_{m, i, j} - 2\vs_m - \vs_{m+1}) / (t_{m+1} - t_m), \nonumber\\
    &\vd_{m, i, j} = (\vs_m + \vs_{m+1} - 2\vm_{m, i, j}) / (t_{m+1} - t_m)^2. \nonumber 
\end{align}
\end{definition}
The original Akima interpolation computes equal weight to the points on both sides, evenly dividing an undulation. When two flat regions with different slopes meet, this modified Akima interpolation~\cite{akima1974method} gives more weight to the side where the slope is closer to zero, thus giving priority to the side that is closer to horizontal, which avoids overshoot. Notice that after pre-computing $\vm_{m, i, j}$ for every edge in $G_A$, every polynomial segment~\cref{eq:akima_segment} can be computed in batch for $G_A$. Furthermore, given a batch of graphs $G_A$, adding a batch dimension for these equations is straightforward. The transition cost is then defined
\begin{equation} \label{eq:cost_akima}
   c(\vq, \vq') = \int_{a}^{b} \left(c_{\textrm{coll}}(f(t)) + 1 \right) \norm{f'(t)} dt,
\end{equation}
where $f(t)$ is the cubic polynomial representing the edge $(\vq, \vq') \in G_A$.

\begin{remark}
    With some algebra derivations, one can verify the cubic polynomial $f_{i, j}(t),\, t \in [t_m, t_{m+1}]$ representing any edge $(\vq_{m, i}, \vq_{m+1, j}) \in G_A$ satisfying four conditions of continuity 
    \begin{align}
        &\vf_{i,j}(t_m) = \vq_{m, i},\,\vf_{i,j}(t_{m+1}) = \vq_{m+1, j}, \\
        &\vf_{i,j}'(t_m) = \vs_{m},\,\vf_{i,j}'(t_{m+1}) = \vs_{m+1}, \nonumber 
    \end{align}
    for any $m \in \{0, \ldots, M+1\},\, i, j \in \{1, \ldots, N\}$. Hence, any path $f \in G_A$ is an Akima spline.
\end{remark}

The Akima spline provides \( C^1 \)-continuity for first-order planning; however, the second derivative is not necessarily continuous. Note that~\cref{thm:prob} does not necessarily hold for Akima Spline Graph $G_A$ and is left for future work.

\section{Theoretical Analysis} \label{sec:theory}


\textbf{Notation.} Let $\gR$ be the set of all paths in $G$. The path cost is the sum of straight-line integrals over the edges $c(g) = \sum_{m=0}^M c(\vq_m, \vq_{m+1}) + c_g(\vg(1)),\, g \in \gR$.
\begin{assumption} \label{asm:uniform}
    We assume that all associated proposal distributions at each layer are uniformly distributed on the configuration space $\forall 1 \leq m \leq M,\, p_m \defi \gU(\gC)$.
\end{assumption}
\begin{assumption} \label{asm:maximum}
    Consider a feasible planning problem, there exists a feasible path $f: [0, 1] \rightarrow \gC_{\textrm{free}}$ having margin $r = \inf_{t \in [0, 1]} \norm{\vf(t) - \vq},\vq \in \gC_{\textrm{coll}}$, such that $r > 0$.
\end{assumption}
These assumptions are common in path planning applications, where the free-path set is not zero-measure $\mu(\gF_{\textrm{free}}) \neq 0$.

\begin{proposition} [Feasibility Check] \label{prop:feasible}
    For any planning problem, $v^*_{G}(\vq_0) < \infty$ if and only if there exists a feasible path in $G$.
\end{proposition}
\begin{proof}
    According to Bellman optimality, $v_G^*(\vq_0) = \min_{g} \{c(g) \mid g \in \gR\}$ is the minimum path cost reaching the goals. By definition, the smoothness term in~\cref{eq:cost} is bounded with $\forall \vq \in \gC$, since $\textrm{TV}(g) < \infty$. Thus, any unbounded path cost $c(\gP) = \infty$ occurs, if and only if $\exists t,\,\vf(t) \in \gC_{\textrm{coll}}$. Hence, $v^*_{G}(\vq_0) =  \min_{\gP}  \{c(\gP) \mid \gP \in \gR\} < \infty$, if and only if $\exists \gP \in \gR,\, c(\gP) < \infty$.
\end{proof}
Proposition \ref{prop:feasible} is useful to filter collided paths after VI.

\begin{lemma}[Solvability In Finite Path Segments] \label{lemma:solve}
    If Assumption~\ref{asm:maximum} holds, there exists a minimum number of segments $M_m \in \sN_{>0}$ for piecewise linear paths to be feasible.
\end{lemma}

\begin{proof}
    We first show that there exists a piecewise linear path $g: [0, 1] \rightarrow \gC$ such that $\norm{f - g}_{\infty} < r$, where $\norm{f - g}_{\infty} = \sup_{t \in [0, 1]} \norm{\vf(t) - \vg(t)}$. 
    We construct $g$ by dividing the interval $[0, 1]$ into $M$ subintervals with length less than $\delta > 0$, i.e., $[t_0, t_1], \ldots, [t_{M-1}, t_{M}]$ with $0 = t_0 < t_1 < \ldots < t_M = 1$. On each subinterval $[t_m, t_{m+1}]$, we define the corresponding segment of $g$ to approximate $f$
    \begin{equation} \label{eq:segment}
    \vg(t) = \vf(t_m) + \frac{\vf(t_{m+1}) - \vf(t_m)}{t_{m+1} - t_m} (t - t_m),\, t \in [t_m, t_{m+1}].
    \end{equation}
    Since by definition the path $f$ is continuous on a compact interval $[0, 1]$, then by Heine-Cantor theorem, $f$ is also uniformly continuous, i.e., $\exists \delta > 0$ for any $a, b \in [0, 1]$, $|a - b| < \delta$, then $\norm{\vf(a) - \vf(b)}_{\infty} < r$. Then, by the construction of $g$ and uniform continuity of $f$, we can choose a $\delta$ sufficiently small such that $\norm{f - g}_{\infty} < r$.
    This implies that there exists a sufficiently large number of segments $M_m$ such that $\delta$ is sufficiently small, hence, $g$ is a feasible path.
\end{proof}

Lemma \ref{lemma:solve} implies that any path planning algorithm producing a piecewise linear feasible path, then it must have a minimum number of segments. 

\begin{lemma} \label{lemma:path_bound}
    Let piecewise linear path $g: [0, 1] \rightarrow \gC$ having $n$ equal subintervals approximating a path $f: [0, 1] \rightarrow \gC$. The error lowerbound is $\norm{f - g}_{\infty} > L / n$, where $L = \textrm{TV}(f)$ is the total variation of $f$.
\end{lemma}

\begin{proof}
    Denoting the subinterval length $u = 1 / n$ and reusing the notations from~\cref{lemma:solve} proof, we define $g$ as a piecewise linear function \ref{eq:segment}. Since $f,g$ are uniformly continuous, the linear interpolation error lower bound can be expressed using the modulus of continuity on a segment $t \in [t_m, t_{m+1}]$
    \begin{equation}
        \norm{\vf(t) - \vg(t)} > \omega_f (u),\, \omega_f (u) = \sup_{|a - b| \leq u} \norm{\vf(a) - \vf(b)}. \nonumber
    \end{equation}
    And, the global error over all segments is
    \begin{equation}
        \norm{\vf(t) - \vg(t)}_{\infty} = \max_{0 \leq m \leq n-1} \sup_{t \in [t_m, t_{m+1}]} \norm{\vf(t) - \vg(t)} \nonumber
    \end{equation}
    By definition, $f$ is uniformly continuous and of bounded variation, the modulus of continuity $ \omega_f(u) $ provides a lower bound for the error on each segment. Therefore, $\norm{f - g}_{\infty} > \omega_f (1 / n)$ on $[0, 1]$. For functions of bounded variation, the modulus of continuity can be bounded in terms of the total variation $\omega_f (1 / n) > L / n$ on $[0, 1]$. Hence, $\norm{f - g}_{\infty} > L / n$. 
\end{proof}

\begin{lemma} \label{lemma:path}
    Let $g_1, g_2$ be a piecewise linear function having the same number of partition points $\{\vg_1(t_m)\}_{m=0}^M,\{\vg_2(t_m)\}_{m=0}^M$ with $0 = t_0 < \ldots, t_M = 1$, $\norm{g_1 - g_2}_{\infty} < \delta$, if and only if $\norm{\vg_1(t_m) - \vg_2(t_m)} < \delta,\, 0 \leq m \leq M$. 
\end{lemma}

\begin{proof}
    \textbf{Sufficiency.} Given $\norm{\vg_1(t_m) - \vg_2(t_m)} < \delta,\,\forall 1 \leq m \leq M$,
    since $g_1, g_2$ are piecewise linear functions, the linear interpolation between partition points $t_m,t_{m+1}$ ensures that the difference between $g_1,g_2$ is maximized at the partition points. Consider $g_1,g_2$ on a segment $[t_m, t_{m+1}]$
    \begin{align}
    \begin{split}
        \norm{\vg_1(t) - \vg_2(t)} \leq &\max \{ \norm{\vg_1(t_m) - \vg_2(t_m)}, \\
        &\norm{\vg_1(t_{m+1}) - \vg_2(t_{m+1})} \} < \delta
    \end{split}
    \end{align}
    Hence, $\norm{g_1 - g_2}_{\infty} = \max_{t \in [0, 1]} \norm{\vg_1(t) - \vg_2(t)} < \delta$.

    \textbf{Necessity.} Given $\norm{g_1 - g_2}_{\infty} < \delta$, then $\norm{\vg_1(t_m) - \vg_2(t_m)} < \delta,\, 0 \leq m \leq M$.
\end{proof}

\begin{theorem}[Probabilistic Completeness] \label{thm:prob}

If Assumption~\ref{asm:uniform} and Assumption~\ref{asm:maximum} hold, for a feasible planning problem $(\gC_{\textrm{free}}, \vq_0, \gG)$, with $G$ having $M \geq M_m$ layers, there exist constants $a, R, L > 0$ depending only on $\gC_{\textrm{free}}$ and $\gG$, such that
\begin{equation}\label{eq:probc}
    \sP\left( v_G^*(\vq_0) < \infty \right) > 1 - M\exp \left(- a\left(R - \frac{L}{M+1} \right)^d N \right).
\end{equation}

\end{theorem}

\begin{proof}

From Lemma \ref{lemma:solve}, if $M \geq M_m$, there exists a feasible piecewise linear path $g$ having $M + 1$ segments with $0 = t_0 < \ldots < t_{M + 1} = 1$ approximating a feasible path $f$. Let $R = \inf_{t \in [0, 1]} \left\{ \norm{\vf(t) - \vq},\vq \in \gC_{\textrm{coll}} 
 \right\}, r = \inf_{t \in [0, 1]} \left\{ \norm{\vg(t) - \vq},\vq \in \gC_{\textrm{coll}} \right\}$ be collision margins of $f,g$, $\gB_{\delta}(\vq) = \{ \norm{\vq' - \vq} < \delta,\, \delta > 0\}$ is an open $\delta$-ball around $\vq$. Now, let $g$ have equal subinterval. 

First, we compute the probability of the event that a sampled graph $G$ has at least a piecewise linear path $h$ with $M+1$ segments such that $h$ is approximating $g$. $h$ is feasible when $\norm{h - g}_{\infty} < r$. We have 
\begin{align}
    \begin{split}
    r &= \inf_{t \in [0, 1]} \left\{ \norm{\vg(t) - \vq},\vq \in \gC_{\textrm{coll}} \right\} \\
    &\leq \inf_{t \in [0, 1]} \left\{ \norm{\vf(t) - \vq},\vq \in \gC_{\textrm{coll}} \right\} + \inf_{t \in [0, 1]} \norm{\vg(t) - \vf(t)} \\
    &= \inf_{t \in [0, 1]} \left\{ \norm{\vf(t) - \vq},\vq \in \gC_{\textrm{coll}} \right\} - \sup_{t \in [0, 1]} \norm{\vg(t) - \vf(t)} \\
    &< R - \frac{L}{M + 1},
    \end{split}
    \nonumber
\end{align}
where the first inequality due to $\gC \subset \sR^d$, and last inequality from Lemma \ref{lemma:path_bound} and $L = \textrm{TV}(f)$.

From Lemma \ref{lemma:path}, since by definition $h, g$ has the same number of segments, the event $\norm{h - g}_{\infty} < r < r_h = R - \frac{L}{M + 1}$ is the event that, given start and goals fixed, for each layer $1 \leq m \leq M$, there is at least one point $\vh(t_m)$ is sampled inside the ball $\gB_{r_h}(\vg(t_m))$. Then, by sampling $N$ points uniformly over $\gC$ per layer (Assumption \ref{asm:uniform}), and the fact that there are pairwise connections between layers, we have the failing probability
\begin{align}
    \begin{split}
    \sP(\norm{h - g}_{\infty} \geq r_h) &\leq \sum_{m=1}^M \left(1 - \frac{\mu(\gB_{r_h}(\vg(t_m))))}{\mu(\gC)} \right)^N \\
    &\leq M\exp \left(- \frac{\alpha_d }{\mu(\gC)} \left(R - \frac{L}{M+1} \right)^d N \right)
    \end{split}\nonumber
\end{align}
where we use the inequality $1 - x \leq e^{-x},\,x \geq 0$, and $a = \alpha_d / \mu(\gC)$, where $\alpha_d$ is the constant term computing volume of a $d$-ball.

The event of $h$ approximating $g$ having equal intervals is a subset of the event of $h$ approximating $g$ having arbitrary intervals. The event that at least a path $h$ in $G$ having $\norm{h - g}_{\infty} < r_h$ is a subset of the event $\exists \textrm{a feasible path in } G$, since there might exist multiple feasible paths and their corresponding piecewise linear approximations have $M+1$ segments. From Proposition \ref{prop:feasible}, $v^*_G(\vq_0) < \infty$ is equivalent to $\exists \textrm{ a feasible path in }G$. We have
\begin{align}
\begin{split}
    \sP(v_G^*(\vq_0) < \infty) &\geq \sP(\norm{h - g}_{\infty} < r_h) \\
    &> 1 - M\exp \left(- a\left(R - \frac{L}{M+1} \right)^d N \right).
\end{split} \nonumber
\end{align}
\end{proof}
The lower bound is intuitive since it directly implies a minimum number of layers $M > [L / R] - 1$ (cf. Lemma~\ref{lemma:solve}) for the exponent coefficient to be strictly positive. It also implies the existence of an optimal number $M^*$; increasing $M$ helps then harms $N$ sample efficiency, depending on the planning problem (cf. Fig.~\ref{fig:sweep_exp}).





\section*{ACKNOWLEDGMENT}

An T. Le was funded by the German Research Foundation project METRIC4IMITATION (PE 2315/11-1). Kay Pompetzki received funding from the German Research Foundation project CHIRON (PE 2315/8-1).

\bibliographystyle{IEEEtran}
\bibliography{IEEEabrv,mp,references}


\end{document}